\documentclass[letterpaper]{article} 
\usepackage{aaai2026}  
\usepackage{times}  
\usepackage{helvet}  
\usepackage{courier}  
\usepackage[hyphens]{url}  
\usepackage{graphicx} 
\usepackage{natbib}  
\usepackage{caption} 
\usepackage{algorithm}
\usepackage{algorithmicx}
\usepackage[noend]{algpseudocode}
\usepackage{xspace}
\usepackage{amsmath}
\usepackage{placeins} 
\usepackage{cleveref}
\usepackage[most]{tcolorbox} 
\usepackage{adjustbox} 
\usepackage{amssymb}
\usepackage{booktabs} 
\usepackage{multirow}
\usepackage{makecell}
\usepackage{amsthm}   
\usepackage{thm-restate}
\theoremstyle{definition}

\tcbset{
  queryexample/.style={
    colback=gray!10!white, 
    colframe=gray!65!white, 
    coltitle=black, 
    fonttitle=\bfseries, 
    sharp corners, 
    boxrule=0.4mm, 
    width=\dimexpr\linewidth-1mm\relax, 
    left=3mm, 
    right=3mm, 
    top=1mm, 
    bottom=1mm, 
    before=\noindent\par\vspace{1mm}, 
    after=\par\vspace{1mm}, 
    fontupper=\small, 
  }
}
\definecolor{codegreen}{rgb}{0,0.6,0}
\definecolor{codegray}{rgb}{0.5,0.5,0.5}
\definecolor{codepurple}{rgb}{0.58,0,0.82}
\definecolor{backcolour}{rgb}{0.95,0.95,0.92}
\definecolor{darkgreen}{rgb}{0.06, 0.64, 0.43} 

\usepackage{newfloat}
\usepackage{listings}
\newcommand{\lineref}[1]{Line~\ref{#1}}

\DeclareCaptionStyle{ruled}{labelfont=normalfont,labelsep=colon,strut=off} 
\floatstyle{ruled}
\newfloat{listing}{tb}{lst}{}
\floatname{listing}{Listing}
\newcommand{\tool}{\texttt{ExPairT-LLM}\xspace}

\title{\tool: Exact Learning for LLM Code Selection by Pairwise Queries{}}
\author{
    Tom Yuviler, Dana Drachsler-Cohen
}
\affiliations{
    Technion, Haifa, Israel\\

    \{tom.yuviler@campus, ddana@ee\}.technion.ac.il
}

\usepackage{bibentry}

\begin{document}

\maketitle

\begin{abstract}
Despite recent advances in LLMs, the task of code generation is still challenging.
To cope, code selection algorithms select the best program from multiple programs generated by an LLM.
However, existing algorithms can fail to identify the correct program,
either because they fail to distinguish nonequivalent programs or because they rely on an LLM and assume it always correctly determines the output for every input.
We present \tool, an exact learning algorithm for code selection that selects a program by posing two new types of queries to an LLM oracle: pairwise membership and pairwise equivalence. These queries are simpler for LLMs and enable \tool to identify the correct program through a tournament, which is robust to some LLM mistakes. We evaluate \tool on four popular code datasets. Its pass@1 (success rate) outperforms the state-of-the-art code selection algorithm on average by $+13.0\%$ and up to $+27.1\%$. It also improves the pass@1 of LLMs performing complex reasoning by $+24.0\%$.
\end{abstract}

\begin{links}
    \link{Code}{https://github.com/TomYuviler/ExPairT-LLM}
    \link{Extended version}{https://arxiv.org/abs/2511.10855}
\end{links}
\section{Introduction}
Large language models (LLMs)
show remarkable performance in code generation~\cite{Code_Llama,HumanEval,NijkampPHTWZSX23,LiAZMKMMALCLZZW23,DeepSeekCoder,LuoX0SGHT0LJ24,FengGTDFGS0LJZ20,Qwen2Coder}. However, they may return code solutions that do not meet the task description, do not align with the input-output examples, or do not even compile~\cite{LiuXW023,Shihan24}.
Code selection algorithms let the LLM generate multiple programs and automatically identify the correct program among them. Some of these approaches rely on input-output examples generated by the LLM~\cite{ChenZNZLLC23,Mouxiang24,ToNB24}; however, the outputs may be incorrect. Other works generate inputs, group programs based on their outputs, and select the cluster containing the most programs~\cite{alpahCode,ShiFGZW22}. Another approach is to use neural networks to estimate the correctness of a program without relying on test inputs~\cite{InalaWYCELMG22,ZhangYHLYF023}. However, these approaches can fail to identify the correct program if the input-output examples are incorrect or if the examples do not differentiate between nonequivalent programs, which can lead to returning an incorrect program. An exception is a method that looks for differentiating inputs through fuzzing and symbolic execution~\cite{FanRMR24}.
However, it is not robust to LLM errors, and fuzzing and symbolic execution tend to fail on tasks with complex input constraints~\cite{FanRMR24}. 

\begin{figure*}[t!]
  \centering
  \includegraphics[width=0.8\linewidth]{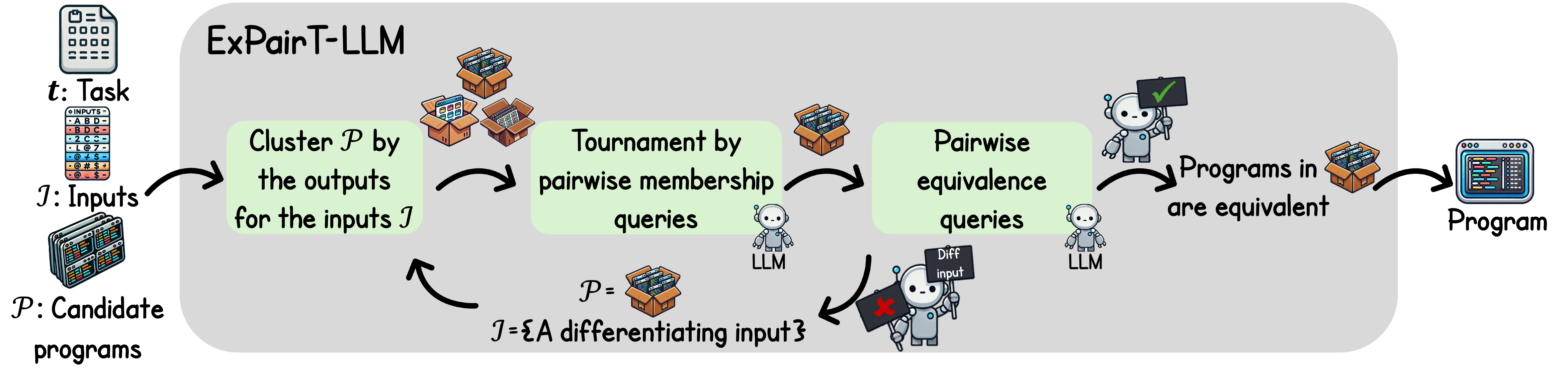}
  \caption{\tool: A code selection algorithm by pairwise queries.}
  \label{fig:overview} 
\end{figure*}

We extend exact learning~\cite{Angluin87b,Angluin87a,Bshouty13,Angluin04} to the problem of LLM code selection. 
We rely on an LLM as the oracle, similarly to prior works for automatically generating input-output examples~\cite{ChenZNZLLC23,Mouxiang24,ToNB24}, evaluating generated text~\cite{Mingqi23,Xinghua23,ChenWJSX23}, and annotating data~\cite{TanLWBJBKL0024,HeLGJZLJYDC24,DingQLCLJB23}.
The challenge is that the traditional queries of exact learning are practically infeasible:
(1)~posing membership queries to the LLM can lead to a very large number 
of queries only to identify the correct output for a \emph{single} input and (2)~equivalence queries are still too challenging for LLMs.
However, LLMs are very successful in pairwise comparisons. 
Previous works use LLMs as judges to decide between two generated texts~\cite{LiusieMG24,LiusieRFG24,Yinhong24,QinJHZWYSLLMWB24} or for chatbot evaluation~\cite{ZhengC00WZL0LXZ23}. 
Recent advancements in 
LLMs' complex reasoning~\cite{Wei0SBIXCLZ22,resoning23,Xuezhi23}
provide the opportunity to use LLMs as judges for the complex task of code selection. 
However, naively selecting a program for a given task by a series of pairwise comparisons over the candidate programs is still too challenging for LLMs.

We introduce new kinds of queries: 
\emph{pairwise equivalence} and \emph{pairwise membership}.
Our queries are simpler for LLMs and enable a learner
to be robust to some oracle mistakes via a tournament~\cite{Trawinski63,Huber63}. A pairwise equivalence query asks whether two programs are equivalent, and if not, the oracle returns a differentiating input. A pairwise membership query asks the oracle to select the better of two sets of outputs for given inputs.
These queries enable identifying correct outputs and constructing a sufficient set of input examples to select a correct program. Further, they enable robustness to some oracle mistakes: differentiating inputs can be validated by running the programs, and the best output
from a set of outputs is robustly determined by
tournaments, which rely on pairwise comparisons.

 We present \tool ({\bf Ex}act learning by {\bf Pair}wise {\bf T}ournament), shown in \Cref{fig:overview}. Given a coding task, initial inputs, and candidate programs, it first clusters the programs based on their outputs for the inputs. It then identifies the correct cluster through a tournament between the clusters, which poses pairwise membership queries to an LLM. 
Next, it checks whether the selected cluster contains only equivalent programs by posing pairwise equivalence queries to the LLM.
If so, it returns one of the programs.
Otherwise,
it refines the cluster based on the differentiating input and repeats this process. If the LLM is always correct, we show that \tool is an exact learner. Otherwise, we show lower bounds on the probabilities that \tool chooses the correct cluster and identifies that a cluster has to be refined. The number of pairwise membership queries and pairwise equivalence queries is at most $|\mathcal{P}| \choose 2$ for each type, where $|\mathcal{P}|$ is the number of candidate programs.

We evaluate \tool on code datasets: \emph{HumanEval}~\cite{HumanEval}, \emph{MBPP-sanitized}~\cite{Austin21}, \emph{APPS}~\cite{hendrycks2021measuring}, and {\emph{LiveCodeBench}}~\cite{jain2025livecodebench}. Its pass@1 (success rate) exceeds $\mathcal{B}^4$~\cite{Mouxiang24}, the state-of-the-art, by $+13.0\%$ and \textsc{CodeT}~\cite{ChenZNZLLC23} by $+16.6\%$.
\tool also improves the pass@1 of LLMs performing complex reasoning: 
by $+32.8\%$ for \emph{OpenAI~o1-mini}~\cite{openaio1mini}, by $+20.4\%$ for \emph{DeepSeek-R1}~\cite{deepseek}, and by $+18.9\%$ for \emph{Gemini 2.5 Flash}~\cite{gemini}.
Both types of pairwise queries are important:
the pass@1 of \tool is higher by $+7.7\%$ as a result of pairwise equivalence queries and 
by $+24.6\%$ as a result of pairwise membership queries.

\section{Related Work}

\paragraph{Code Selection.}
Several code selection algorithms rely on LLMs to generate input-output examples and cluster programs by their consistency with the examples~\cite{ChenZNZLLC23,Mouxiang24,ToNB24}. Other approaches generate inputs, group programs by their outputs, and select the cluster with the most programs~\cite{alpahCode,ShiFGZW22}.
However, they can fail for various reasons: if the input-output examples are incorrect, if the selected cluster is incorrect, or if the selected cluster contains nonequivalent programs.
\texttt{LLMCodeChoice}~uses fuzzing and symbolic execution to find inputs that differentiate candidate programs~\cite{FanRMR24}; however, it is unsuitable for tasks with complex input constraints and assumes the LLM is always correct. Others use neural networks to estimate the correctness of a program for a task without executing the program on test inputs~\cite{InalaWYCELMG22,ZhangYHLYF023}.

\paragraph{LLM as an Oracle.}  
Relying on LLMs as oracles (judges) for selection between two candidates has been proposed for text-related tasks, including summarization~\cite{LiusieMG24,LiusieRFG24,Yinhong24}, chat assistants~\cite{LiusieMG24,LiusieRFG24,ZhengC00WZL0LXZ23}, content generation~\cite{Yinhong24}, and information retrieval~\cite{QinJHZWYSLLMWB24}.  
CodeRL~\cite{Le0GSH22} employs reinforcement learning to enhance code-generating LLMs during training and inference by utilizing a separate LLM trained to assess the functional correctness of generated code samples.  
Other works utilize LLMs for self-repair in code generation models, providing feedback that identifies errors from execution results and generates explanations to fix the code~\cite{ChenLSZ24,ZhangLLLJ23,OlaussonIW0S24}.

\paragraph{Program Synthesis.}   
Our work is related to 
\emph{oracle-guided inductive synthesis} (OGIS)~\cite{JhaS17,JiLXZH20,JiKXH23}, which 
generates a program meeting a specification by 
 interacting with a SAT/SMT oracle. 
In particular, it is related to \emph{counterexample-guided inductive synthesis} (CEGIS)~\cite{Solar-LezamaTBSS06,AlurBJMRSSSTU13,AbateDKKP18}, where if the oracle eliminates a candidate program, it also returns a counterexample.

\section{Problem Definition}

\paragraph{Tasks and Programs.}
A coding task consists of a natural language description (e.g., a docstring describing the function) and possibly input-output examples. \Cref{fig:tasks} shows task examples. If a program $p$ satisfies a task $t$, we say $p$ is correct and write $p\models t$. For each task, we are given a finite set of \emph{candidate programs} $\mathcal{P}$. We assume that there exists $p\in \mathcal{P}$ that is correct and that the programs are deterministic and terminate for every input (stochastic or nonterminating programs can be removed in a preprocessing step by running them on several inputs with a reasonable timeout). Obtaining the set of candidates is orthogonal to our problem.

\begin{figure}[tb]
  \centering
  \includegraphics[width=\linewidth]{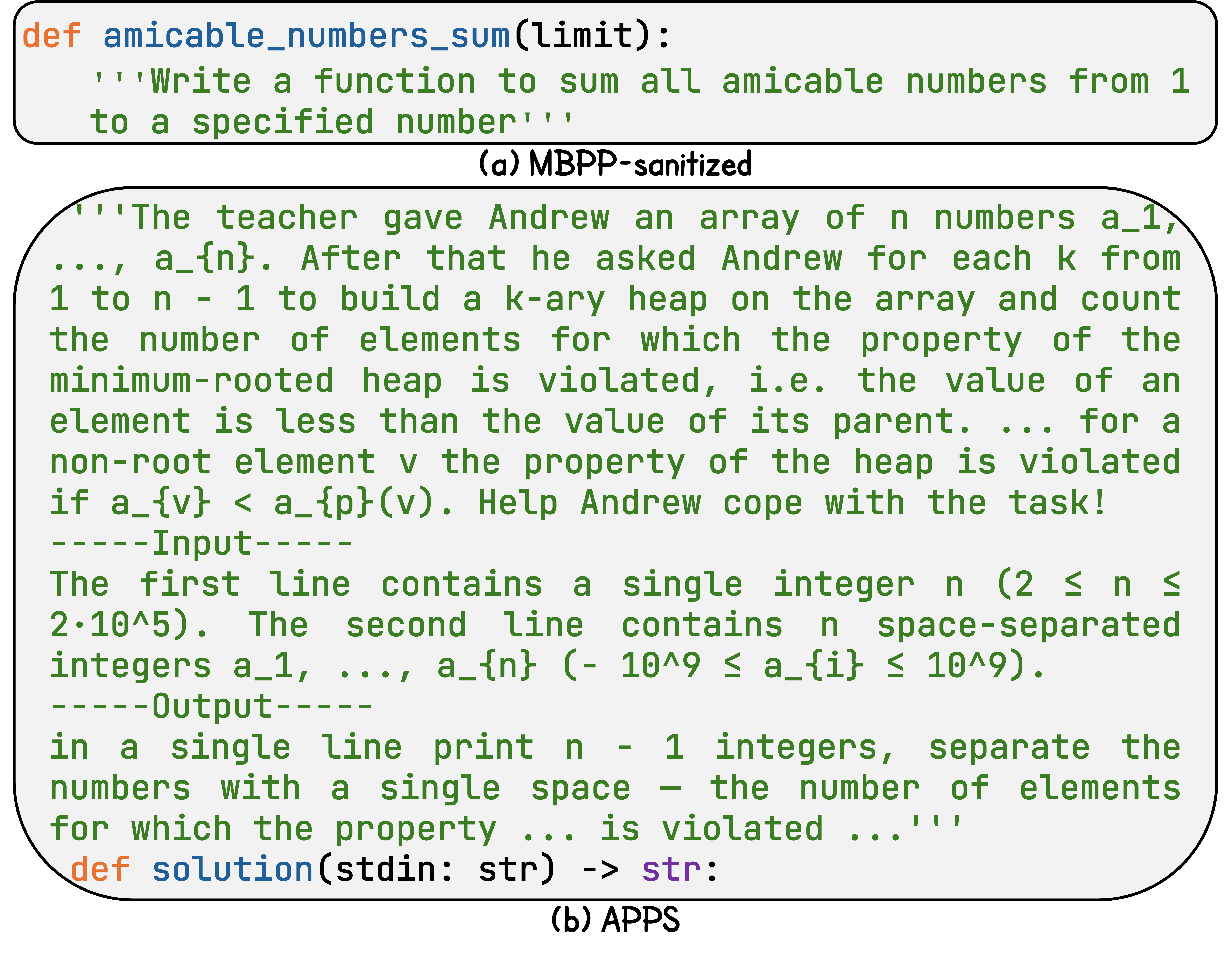}
  \caption{Tasks from (a)~\emph{MBPP-sanitized}~\cite{Austin21}, and (b)~\emph{APPS}~\cite{hendrycks2021measuring}.}
  \label{fig:tasks}
\end{figure}

\paragraph{Code Selection.} Given a task~$t$ and a set of programs $\mathcal{P}$ that contains a correct program, a code selection algorithm $\mathcal{A}$ aims to return $p\in \mathcal{P}$ satisfying $t$. We evaluate algorithm $\mathcal{A}$ by its pass@1~\cite{HumanEval} (success rate).
Given a set of task-program-set pairs $\mathcal{T}$, the pass@1 of $\mathcal{A}$ is the percentage of tasks for which it returns a correct program:
\begin{equation}
    \text{pass@1}(\mathcal{A}) = \frac{100}{|\mathcal{T}|} \sum_{(t_i, \mathcal{P}_i) \in \mathcal{T}} 
    \mathbb{I}
    \left[
         \mathcal{A}(t_i, \mathcal{P}_i) \models t_i 
   \right]
    \label{eq:success_rate}
\end{equation}

\paragraph{Challenges.}
This problem is challenging. First, tasks are described by natural language and examples, which may be ambiguous, under-specified, or require domain knowledge, e.g., mathematics (\Cref{fig:tasks}(a)). Second, the candidate programs are often too complex~\cite{Shihan24} to be analyzed by existing program analyzers, especially if the tasks pose complex input constraints (e.g.,~\Cref{fig:tasks}(b)), which are not easily expressed with existing tools. Third, the algorithm cannot interact with a user to obtain more information. Our problem is not a classical program synthesis problem~\cite{GulwaniPS17,AlurBJMRSSSTU13}. First, the program space can be complex, consisting of hundreds of lines of code, spanning a wide range of operations, data structures, and sometimes uncommon libraries. Second, the task may include a lengthy natural language description that is difficult to translate into a formal specification. Third, the tasks can vary significantly, e.g., mathematics (\Cref{fig:tasks}(a)) and algorithms and data structures (\Cref{fig:tasks}(b)), whereas program synthesizers often focus on related tasks. 
\section{Pairwise Membership \& Equivalence Queries}

In this section, we present our idea: exact learning by posing new queries to an oracle.

\paragraph{Exact Learning.}
In exact learning~\cite{Angluin87b,Angluin87a}, a learner identifies a target concept (a set) by interacting with an oracle. Typically, the learner poses two types of queries. A membership query asks whether an element is a member of the target, and the oracle replies with \emph{yes} or \emph{no}. An equivalence query asks if a candidate concept is equivalent to the target. If the oracle replies \emph{yes}, the learner terminates. Otherwise, the oracle returns a counterexample element. Often, exact learning algorithms identify the target concept if it is in the search space and the oracle answers correctly.
In our setting, a membership query asks whether an input-output pair aligns with the given task, and an equivalence query asks whether a program satisfies the given task. However, membership queries may require an excessively large number of interactions to determine the correct output for a single input, and equivalence queries are too challenging for LLMs.

\paragraph{Pairwise Comparisons.}
LLMs have been shown to be effective as judges for selecting between two candidates in text-related tasks~\cite{LiusieMG24,LiusieRFG24,Yinhong24,QinJHZWYSLLMWB24,ZhengC00WZL0LXZ23}. Additionally, several works rely on 
pairwise comparisons to identify the best candidate from a set of candidates~\cite{WauthierJJ13,JamiesonN11,ShahW17,Cattelan12,Newman22}. One robust approach is \emph{Copeland's method}~\cite{Cop51}, a voting system that determines the winner by head-to-head comparisons between each pair of candidates. 
In every comparison, one candidate gains a point. The winner is the candidate with the most points. Theoretically, we could rely on Copeland's method to identify a correct program by posing to the LLM oracle a series of queries, each of which presents two programs and asks the oracle to select the better program for the given task. However, such pairwise queries are still challenging for LLMs.

\paragraph{Pairwise Queries.}
Recent advancements have significantly enhanced LLM capabilities in complex reasoning~\cite{Wei0SBIXCLZ22,resoning23,Xuezhi23}, leading to remarkable performance in coding problems. While the previous query types are challenging for them, we observe that LLMs can reliably answer \emph{pairwise membership} and \emph{pairwise equivalence} queries. A pairwise membership query consists of a task, a list of $k$ inputs, and two lists of $k$ outputs $(t, \mathcal{I},O_1,O_2)$, asking which outputs are more suitable for $t$ and $\mathcal{I}$: $O_1$ or $O_2$.
A pairwise equivalence query consists of a task and two programs $(t, p_1,p_2)$, asking whether the programs are semantically equivalent
with respect to $t$. The oracle replies with \emph{yes} or \emph{no}. If it replies \emph{no}, it provides a differentiating input \texttt{x}, i.e., $p_1(\texttt{x})\neq p_2(\texttt{x})$. As an example, consider the task $t$ of computing the length of a string \texttt{s}. A pairwise membership query is \texttt{($t$,[`Banana'],[5],[6])}, whose answer is \texttt{[6]}. A pairwise equivalence query is $(t, \texttt{return length(s)}, \texttt{return 6})$. The answer is \emph{no}, and a possible counterexample is \texttt{`Apple'} (the first program returns \texttt{5} and the second program returns~\texttt{6}). \section{Exact Learning by Pairwise Queries to LLMs}\label{sec:slect}
In this section, we introduce \tool.

\begin{algorithm}[tb]
  \caption{\tool($t$, $\mathcal{I}$, $\mathcal{P}$, $\mathcal{L}$)}
  \label{alg:selecting}
  \begin{algorithmic}[1]
      \Statex\texttt{\bf Input: } Task $t$, inputs $\mathcal{I}$, programs $\mathcal{P}$, and LLM $\mathcal{L}$.
    \Statex\texttt{\bf Output: } A program $p \in \mathcal{P}$.
    \State $C^* = \mathcal{P}$
      \Comment{The selected cluster; initially all programs}\label{ln:init}
    \While{\texttt{True}}\label{ln:while}
      \State     $\mathcal{C},\mathcal{O} = \texttt{cluster}(\mathcal{I}, {C}^*)$ \Comment{Clusters and outputs}\label{ln:div}
      \State $\texttt{scores} = [0] \times |\mathcal{C}|$ \Comment{Zero scores for clusters}\label{ln:counts_init}
    \For{$i = 1,\dots,|\mathcal{C}|$}\label{ln:smem}\Comment{Tournament}
        \For{$j = i+1,\dots,|\mathcal{C}|$}
      \State $\texttt{out} = \mathcal{L}.\mathtt{membership}(t, \mathcal{I}, \mathcal{O}[i], \mathcal{O}[j])$
\If{$\texttt{out} == \mathcal{O}[i]$} $\texttt{scores}[i] \mathrel{+}= 1$
\Else\  $\texttt{scores}[j] \mathrel{+}= 1$\label{ln:emem}
\EndIf
              \EndFor

      \EndFor
       \State ${C}^* = \mathcal{C}[\text{argmax}(\texttt{scores})]$ \Comment{Selected cluster}\label{ln:argmax}
       \State  $\texttt{eq} = \texttt{True}$\label{ln:flagf}
       \For{$i = 2,\dots,|{C}^*|$}\label{ln:sequ}
        \State         $\texttt{x} = \mathcal{L}.\mathtt{equivalence}(t, {C}^*[1], {C}^*[i])$\label{ln:eequ}
        \If{$\mathtt{x} \neq \bot \;\textbf{and}\; {C}^*[1](\mathtt{x}) \neq {C}^*[i](\mathtt{x})$}\label{ln:sfoundfiff}
        \State $\mathcal{I}$ = $[ \texttt{x}];\texttt{eq} = \texttt{False};\textbf{Break}$\label{ln:efoundfiff}
        \EndIf
       \EndFor
    \If{$\mathtt{eq}$}
    \Return ${C}^*[1]$\label{ln:sreturn_1}
    \EndIf
    
    \EndWhile
    \State \texttt{\bf Function }\texttt{cluster($\mathcal{I}$, $\mathcal{P}$)}\label{alg:clustering}
    \Statex\texttt{\bf Input: } Inputs $\mathcal{I}$ and candidate programs $\mathcal{P}$.
    \Statex\texttt{\bf Output: } Clusters of programs $\mathcal{C}$ and outputs $\mathcal{O}$.
    \State $\mathcal{C} = \texttt{[]}$;
$\mathcal{O} = \texttt{[]}$ \Comment{Lists of clusters and their outputs}\label{ln:clusteringb}
  \ForAll{$p \in \mathcal{P}$}
    \State     $\texttt{outputs} = [p(i) \mid i \in \mathcal{I}]$ \Comment{The outputs of $p$}
    \State $\texttt{found} = \texttt{False}$
    \For{$i = 1,\dots,|\mathcal{C}|$}
    \If{$\mathtt{outputs == }\ \mathcal{O}\mathtt{[i]}$}
     \State $\mathcal{C}\texttt{[i].append}(p);
            \texttt{found=True};
            \textbf{Break}$ 
    \EndIf
    \EndFor
    \If{$\textbf{not } \mathtt{found}$}
        \State $\mathcal{C}\texttt{.append([p])};
        \mathcal{O}\texttt{.append(outputs)}$
    \EndIf

  \EndFor    
  \State \Return $\mathcal{C}$, $\mathcal{O}$\label{ln:clusteringe}
  \end{algorithmic}
\end{algorithm}

\paragraph{Pseudocode.}  
\Cref{alg:selecting} shows its pseudocode. 
The inputs are a coding task in natural language \(t\), 
a list of inputs \(\mathcal{I}\) (provided by the user or LLM-generated), 
a list of candidate programs \(\mathcal{P}\), and an LLM \(\mathcal{L}\). It returns a program \(p \in \mathcal{P}\), intended to satisfy \(t\).
It first initializes the selected cluster \({C}^*\) to \(\mathcal{P}\) (\lineref{ln:init}). Next, it iteratively refines \({C}^*\) until \({C}^*\) contains only equivalent programs (\lineref{ln:while}). 
Each iteration begins by invoking the \texttt{cluster} function (\lineref{ln:div}). This function partitions \({C}^*\) into clusters $\mathcal{C}$ and their respective outputs $\mathcal{O}$, based on the outputs of programs in \({C}^*\) for the inputs~$\mathcal{I}$ (\lineref{ln:clusteringb}--\lineref{ln:clusteringe}). Then, \tool initializes the \texttt{scores} array, which records the clusters' scores (\lineref{ln:counts_init}).
Then, Copeland's method is executed by posing a pairwise membership query for each pair of clusters (\lineref{ln:smem}--\lineref{ln:emem}). The cluster with the maximal score is set to $C^*$ (\lineref{ln:argmax}).
Next, \tool checks if the programs in $C^*$ are equivalent.
It initializes the \texttt{eq} flag to \texttt{True} (\lineref{ln:flagf}).
 It then poses pairwise equivalence queries to the LLM between the first program in $C^*$ and every other program in \({C}^*\) (\lineref{ln:sequ}--\lineref{ln:eequ}). 
 If the LLM finds a differentiating input \texttt{x} for two programs, \tool validates it by executing both programs on \texttt{x} (\lineref{ln:sfoundfiff}). 
If the validation succeeds, \tool sets the inputs \(\mathcal{I}\) to [\texttt{x}], negates the flag \texttt{eq}, and begins another iteration to refine~$C^*$ (\lineref{ln:efoundfiff}).
Otherwise, if \texttt{eq} remains \texttt{True}, \tool returns the first program in~$C^*$ (\lineref{ln:sreturn_1}).
The appendix shows the query prompts.

\paragraph{Example.}
\Cref{fig:example} shows a running example, for the task of computing the length of a string \texttt{s}, the inputs $\mathcal{I}=[\texttt{`Banana'}]$, and 
four candidate programs (\Cref{fig:example}(a)). \tool runs every program on the input and clusters the programs based on their outputs (\Cref{fig:example}(b)): 
 $A$ (for \texttt{6}), $B$ (for \texttt{5}), and $C$ (for \texttt{4}). 
\tool identifies the correct cluster by three pairwise membership queries to determine the most suitable output for \texttt{`Banana'} (\Cref{fig:example}(c)).
Accordingly, it selects cluster~$A$. Note that when two incorrect outputs are compared (e.g., \texttt{5} and \texttt{4}), the oracle selects one of them.
 Next, it poses a pairwise equivalence query to the LLM for the two programs in~$A$ (\Cref{fig:example}(d)). Since the programs are not equivalent, the LLM returns a differentiating input \texttt{`Apple'}.
\tool then refines cluster~$A$ into two clusters~$A_1$ and $A_2$ based on the output for \texttt{`Apple'} (\Cref{fig:example}(e)). 
It then selects the correct cluster by a pairwise membership query to determine the more suitable output for \texttt{`Apple'}  (\Cref{fig:example}(f)). Accordingly, it selects~$A_1$. Since~$A_1$ contains only one program, \tool returns this program (\Cref{fig:example}(g)).

\begin{figure*}[t!]
  \centering
  \includegraphics[width=\linewidth]{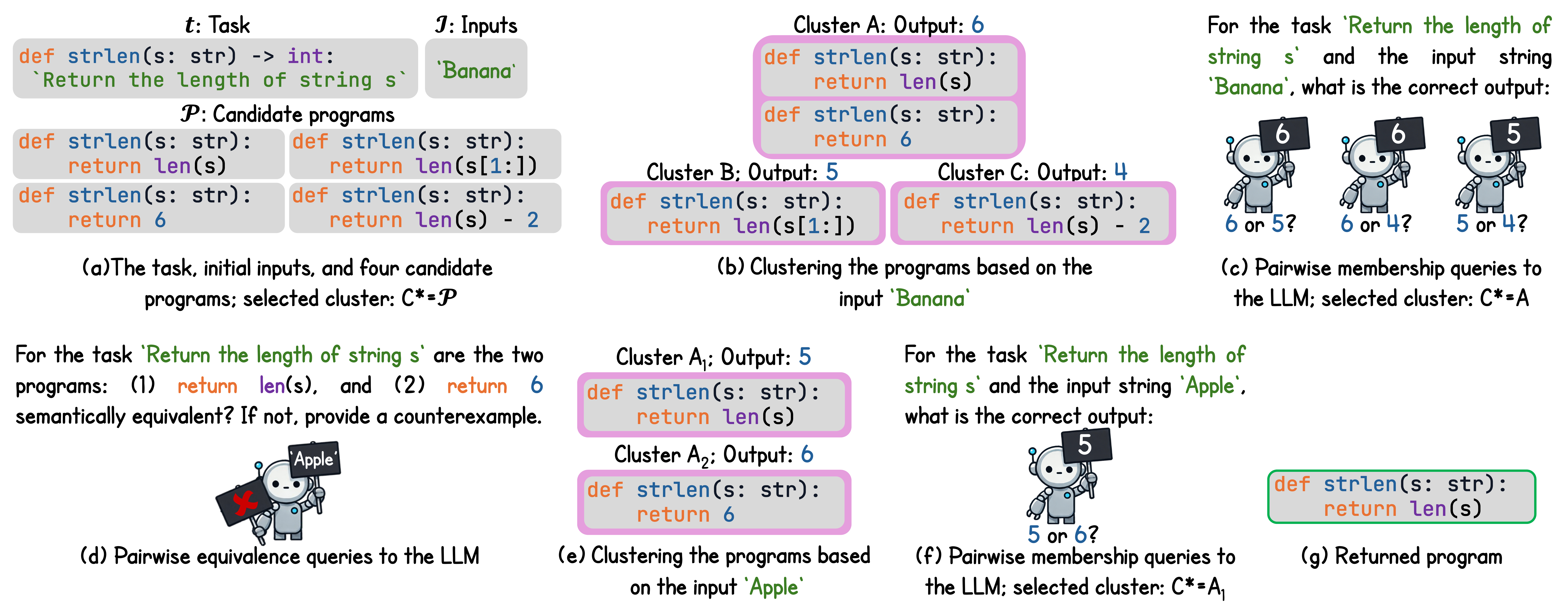}
  \caption{A running example of \tool.}
  \label{fig:example}
\end{figure*}

\paragraph{Advantages.}
\tool has several advantages. First, unlike approaches that return a program from the first selected cluster, 
it checks if the selected cluster contains nonequivalent programs. 
Second, unlike most approaches where the input examples are not adapted to the candidate programs, 
it generates inputs for distinguishing nonequivalent programs. 
Third, unlike approaches in which a single incorrect answer by the LLM can lead to failure,
\tool is robust to some LLM errors: (1)~it selects a cluster by the aggregated 
score in Copeland's method, and (2)~it validates that differentiating inputs are indeed differentiating by running the programs.
Fourth, unlike approaches that rely on fuzzing and symbolic execution, it can handle tasks with complex input constraints.
Fifth, it leverages recent advancements enabling LLMs to perform exceptionally well in pairwise 
comparisons~\cite{LiusieMG24,LiusieRFG24,Yinhong24,QinJHZWYSLLMWB24}.

\paragraph{Limitations.}
Since \tool compares the outputs of candidate programs, it targets tasks that can be implemented as stateless functions, common in many widely used code generation benchmarks~\cite{HumanEval,Austin21,hendrycks2021measuring,jain2025livecodebench}. 
It is not designed for programs that do not take any input (e.g., retrieving static configuration settings), do not generate explicit outputs (e.g., updating storage without direct feedback), or exhibit nondeterminism or nontermination (relevant for specialized code generation benchmarks, e.g., DS-1000~\cite{dsLai0WZZZYFWY23} and \textsc{BigCodeBench}~\cite{ZhuoVCH0WYZHPB025}).

\paragraph{Guarantees.}\label{subsec:analysis}
We next present theoretical guarantees and proof sketches. 
Full proofs are in the appendix. 
Recall that the program space $\mathcal{P}$ is finite, contains a correct program, and consists of terminating and deterministic programs. 
We first show that \tool terminates.
Second, if the oracle  $\mathcal{L}$ is accurate, \tool returns the correct program. Otherwise, we show a lower bound on the probability of it choosing the correct cluster and show the probability of identifying that a cluster contains nonequivalent programs.
Lastly, we analyze the number of queries.

\begin{restatable}[]{lemma}{fta}
    \label{lem:termination}
    \Cref{alg:selecting} terminates within $|\mathcal{P}|$ iterations.
    \end{restatable}
It follows since the size of the selected cluster $C^*$ decreases at every iteration, starting from the second one.  

We next assume an accurate oracle that answers correctly.

\begin{restatable}[]{lemma}{ftb}
    \label{lem:mem}
    If $\mathcal{L}$ is accurate, at every iteration (\lineref{ln:while}), a correct program $p^* \in \mathcal{P}$ is in the selected cluster ${C}^*$.
\end{restatable}
The proof is by induction. Base: $C^*=\mathcal{P}$. Step: since $\mathcal{L}$ is accurate, the cluster containing $p^*$ obtains the maximal score in Copeland's method and is thus selected as $C^*$.

\begin{restatable}[]{lemma}{ftc}
    \label{lem:equ}
    If $\mathcal{L}$ is accurate, at any iteration that ${C}^*$ contains an incorrect program, $\mathcal{L}$ returns a differentiating input.
\end{restatable}
By the assumption and~\Cref{lem:mem}, $C^*$ has an incorrect program $p'$ and a correct one $p^*$. 
\tool poses a pairwise equivalence query for the first program in $C^*$ and every other program in $C^*$. Since $\mathcal{L}$ is accurate and $p'\not \equiv p^*$, $\mathcal{L}$ must return a differentiating input.

\begin{restatable}[]{theorem}{ftd}
\label{thm:correct}
If $\mathcal{L}$ is accurate, \tool returns a correct program.
\end{restatable}
Since \tool terminates, it returns a program. By~\Cref{lem:mem} and~\Cref{lem:equ}, it must be a correct program.

Next, we assume an inaccurate oracle. We begin with a lower bound on the probability that, at any iteration, $C^*$ is set to the cluster containing a correct program out of all clusters $\mathcal{C}=\{C_1,\ldots,C_{n+1}\}$ (\lineref{ln:argmax}). Without loss of generality, we assume it is $C_{n+1}$. Since $C^*$ is set to the cluster with the maximal score in Copeland's method, we consider the event \(A\) that $C_{n+1}$ has the maximal score
and provide a lower bound on its probability.
We make two assumptions, common in tournaments~\cite{Trawinski63,Huber63}. First, the oracle's responses are independent. Second, the probability that the 
oracle's response to a pairwise membership over (the outputs of) $C_i$ and $C_j$ is correct equals: (1)~$p$, for $p>0.5$, if 
 $C_i$ or $C_j$ is $C_{n+1}$ or (2)~$0.5$, otherwise.
 These assumptions hold in practice. First, queries are independent API calls to the LLM, which has no memory between them. Second, to illustrate that $p>0.5$ holds in practice, we measure the proportion of correct responses of \emph{OpenAI o1-mini}~\cite{openaio1mini} in pairwise membership queries over programs generated by \emph{CodeLlama}~\cite{Code_Llama} and \emph{DeepSeek-Coder}~\cite{DeepSeekCoder} on the \emph{HumanEval} dataset~\cite{HumanEval}. The results estimate $p=0.87$ for \emph{CodeLlama} and $p=0.96$ for \emph{DeepSeek-Coder}. 
 
 \begin{figure}[tb]
  \centering
  \includegraphics[width=0.8\linewidth]{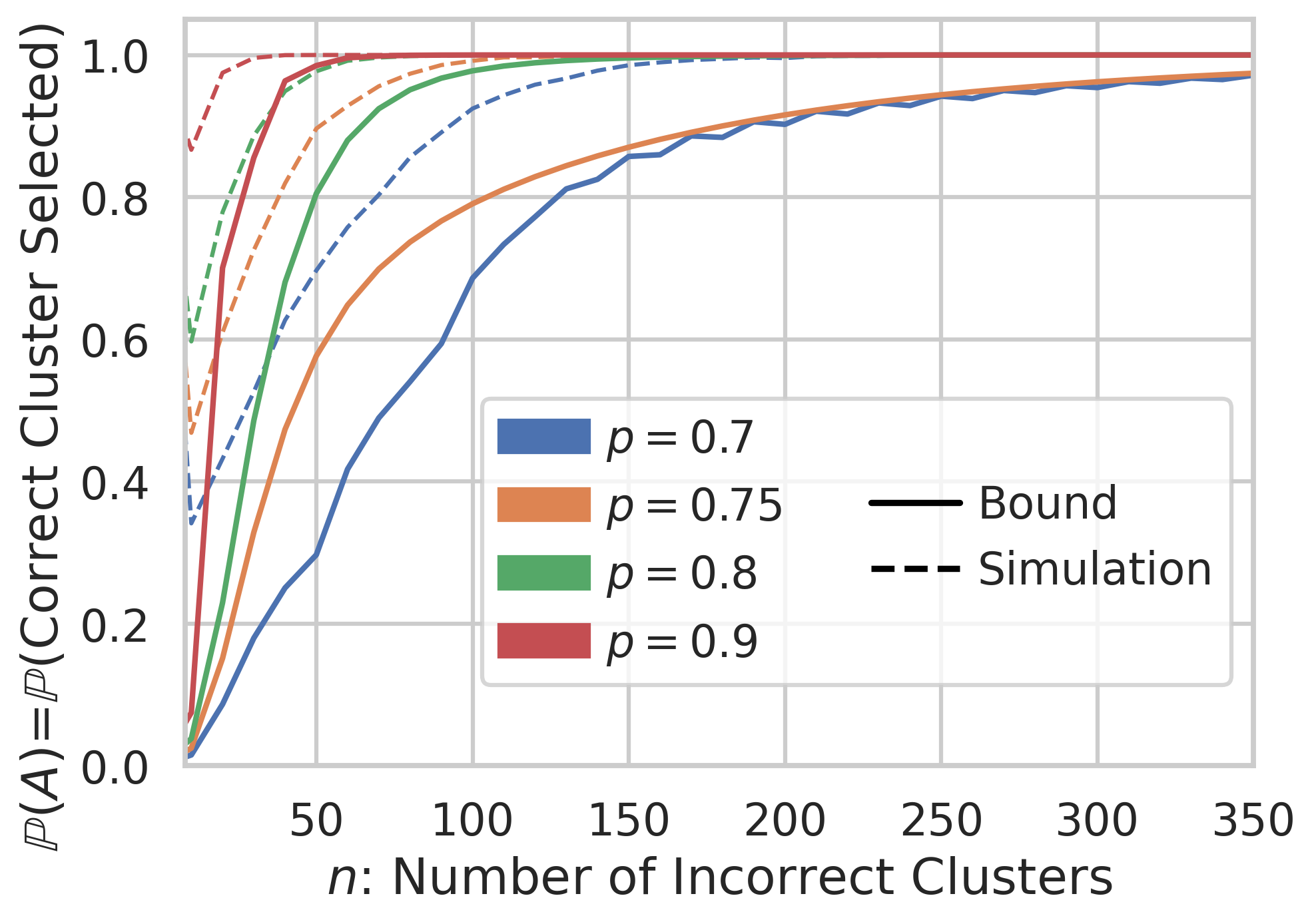}
  \caption{The empirical probability (dashed) and our lower bound (solid) for different \(p\).}
  \label{fig:lower_bound2}
\end{figure}

\begin{restatable}[]{theorem}{fth}
\label{thm:clusters_bound}
Under the tournament's assumptions, a lower bound on the probability that the correct cluster $C_{n+1}$ has the maximal score in~\lineref{ln:argmax} is:
\begin{align*}
  \mathbb{P}(A)
  &
  \ge
    \biggl[\sum_{k=\lceil j\rceil+2}^{n}\binom{n}{k}\,p^k\,q^{\,n-k}\biggr]\\
  &\quad
  \cdot
    \biggl[1 - n\Bigl(\tfrac12\Bigr)^{n-1}
      \sum_{k=\lfloor j\rfloor+1}^{\,n-1}\binom{n-1}{k}
    \biggr],
\end{align*}
where $q\triangleq 1-p$ 
 and $j\in  [1,n-1]$ is a real-valued parameter. 
\end{restatable}
The proof defines the event $B$ that the maximum score over all clusters but the correct one is at most~$j$, and 
lower bounds $\mathbb{P}(A)$ with $\mathbb{P}(A|B)\cdot \mathbb{P}(B)$. It then shows that $\mathbb{P}(A|B)$ and $\mathbb{P}(\overline{B})$ are sums over binomially distributed random variables with probabilities $p$ and $0.5$, respectively. 
\Cref{fig:lower_bound2} compares the empirical \(\mathbb{P}(A)\) to the lower bound as a function of $n$, 
for different values of $p$ and $j\in [\frac{n-1}{2},n-1]$. 
\(\mathbb{P}(A)\) is computed by simulations with \(10^4\) repetitions for every \(n\) and \(p\). 
The figure shows that the larger the $n$, the tighter our lower bound and that  \(\mathbb{P}(A) \ge 0.89\) for \(p \geq 0.9\) and \(n \geq 25\).

Next, we show the probability that, at any iteration, \tool finds a differentiating input 
if $C^*$ has nonequivalent programs. We assume that the oracle's responses to pairwise equivalence queries are independent, and that the probability that it returns a differentiating input for a pairwise equivalence query is $p$ if one of the programs is correct and the other one is incorrect, or $p'$ if both programs are incorrect. 
In the following, we denote by \(d\in[|C^*|]\) the number of incorrect programs in \({C}^*\) and by \(A\) the event of \tool finding a differentiating input.  

\begin{restatable}[]{theorem}{ftg}
\label{thm:diff_found}
The probability of \tool finding a differentiating input is:
\begin{align*}
\mathbb{P}(A)
&=
\Bigl(1 - (1-p)^{d}\Bigr)\,\frac{\lvert C^*\rvert - d}{\lvert C^*\rvert}
\\
&\quad+
\Bigl(1 - (1-p)^{\lvert C^*\rvert - d}\,(1-p')^{d-1}\Bigr)\,\frac{d}{\lvert C^*\rvert}.
\end{align*}
\end{restatable}

\begin{figure}[t]
  \centering
  \includegraphics[width=\linewidth]{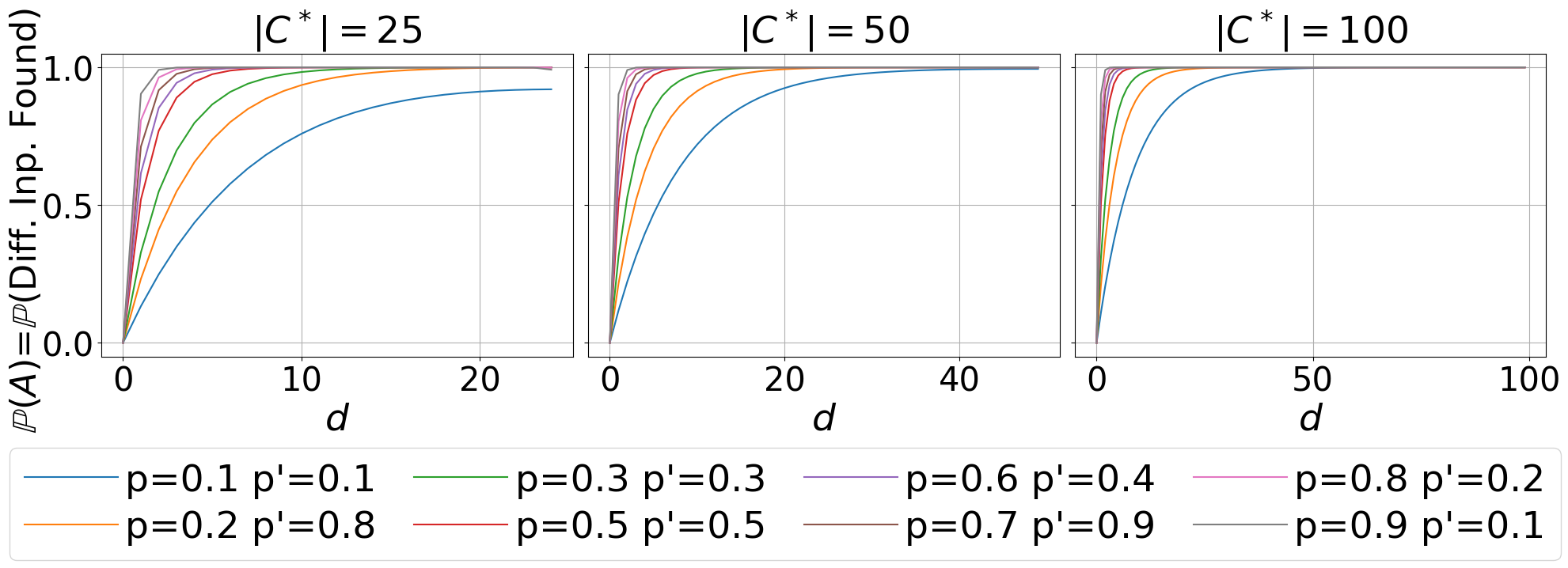}
  \caption{The probability of finding a differentiating input 
    as a function of the number of incorrect programs $d$. 
}
  \label{fig:prob_find_diff2}
\end{figure}

The proof defines the event $B$ that the first program in $C^*$ is correct and shows expressions for $\mathbb{P}(A|B)\cdot \mathbb{P}(B)+\mathbb{P}(A|\overline{B})\cdot \mathbb{P}(\overline{B})$, which yield the above $\mathbb{P}(A)$.
\Cref{fig:prob_find_diff2} shows \(\mathbb{P}(A)\) as a function of the number of incorrect programs $d$, 
for different values of $|C^*|$, $p$ and $p'$. For $p\geq 0.5$ and $d\geq 5$ (even for low $p'$), 
the probability of \tool identifying that $C^*$ has incorrect programs is high ($>0.9$).
Namely, if the selected cluster has a correct program, the probability of returning an incorrect program is low.
Empirically, $p \ge 0.5$ in practical settings. For example, on \emph{HumanEval}, with \emph{OpenAI o1-mini} as the oracle, we observe $p=0.72$ for \emph{CodeLlama} and $p=0.61$ for \emph{DeepSeek-Coder}.

Lastly, we show that the maximal number of queries of each type is \(|\mathcal{P}| \choose 2\). This polynomial complexity is consistent with exact learners in other 
domains~\cite{HermoO20,MarusicW15,HellersteinPRW96}.

\begin{restatable}[]{theorem}{fte}
\label{thm:memb_worst_case}\label{thm:equiv_worst_case}
\tool poses up to $|\mathcal{P}| \choose 2$ pairwise membership queries and $|\mathcal{P}| \choose 2$ pairwise equivalence queries.
\end{restatable}

The proof shows that each pair of programs can be compared by each type of query at most once.

\section{Evaluation}\label{sec:eval}
Here, we evaluate \tool and show that:
(1)~it outperforms existing code selection algorithms, including the state-of-the-art, 
(2)~it significantly improves the pass@1 of LLMs with complex reasoning capabilities, 
(3)~both pairwise membership and pairwise equivalence queries improve its performance,
and (4)~it poses 25.7 queries on average.

\paragraph{Setup.}
We implemented \tool in Python. The experiments were run on an Ubuntu 20.04.6 OS on a dual AMD EPYC 7742 server with 1TB RAM. Unless otherwise specified, each task has \(|\mathcal{P}| = 25\) candidate programs, an initial input list \(\mathcal{I}\) with five inputs generated by \emph{GPT-4o}~\cite{GPT4}, and \emph{OpenAI o1-mini}~\cite{openaio1mini} as the LLM oracle.
We evaluate on three datasets: \emph{HumanEval}~\cite{HumanEval}, which has 164 coding tasks; \emph{MBPP-sanitized}~\cite{Austin21} (denoted as \emph{MBPP}), which has 427 coding tasks; and \emph{APPS}~\cite{hendrycks2021measuring}, whose test set has 5,000 coding tasks.
To fairly compare with prior works~\cite{Mouxiang24,ChenZNZLLC23}, for \emph{HumanEval}, we exclude input-output examples provided within the tasks.
These datasets are the most popular benchmarks for evaluating code generation models~\cite{Code_Llama,DeepSeekCoder,Qwen2Coder}. 
Each of these datasets provides test input-output examples for each task (hidden during the selection), which we use to evaluate the correctness of the selected programs.
For each dataset, we consider several models for generating the program space $\mathcal{P}$:
\emph{Codex}~\cite{HumanEval} (code-davinci-002 version), \emph{CodeLlama}~\cite{Code_Llama} (7B-Python version), \emph{StarCoder}~\cite{LiAZMKMMALCLZZW23}, \emph{DeepSeek-Coder}~\cite{DeepSeekCoder} (6.7B Instruct version), and \emph{CodeGen}~\cite{NijkampPHTWZSX23} (MONO-16.1B version). 
Due to the substantially larger number of coding tasks in \emph{APPS}, we follow previous works~\cite{Mouxiang24,ChenZNZLLC23} and consider for its evaluation only the \emph{Codex} model.
The program space generated by each model is obtained from the repository provided by~\citet{Mouxiang24}.
We evaluate a code selection algorithm by the popular pass@1 metric (\Cref{eq:success_rate}). 
Like~\citet{Mouxiang24,ChenZNZLLC23}, we exclude from the pass@1 pairs $(t, \mathcal{P})$ where all programs in $\mathcal{P}$ satisfy $t$, or where no program in $\mathcal{P}$ satisfies $t$.

\paragraph{Baseline Comparisons.}
We compare \tool with two code selection algorithms. \emph{First,} $\mathcal{B}^4$~\cite{Mouxiang24}, the state-of-the-art, defines
the optimal selection strategy using a Bayesian framework, where the posterior probability of passing states between solutions and tests guides the selection. It approximates this posterior using Bayesian statistics and reformulates the problem as integer programming. $\mathcal{B}^4$ has three variants. 
In our comparison, for each dataset and model, we define its pass@1 by the best variant.
 \emph{Second,} \textsc{CodeT}~\cite{ChenZNZLLC23}, which generates test cases for each task, executes the candidate programs on these cases, and selects a program based on dual execution agreement. This agreement considers both the consistency of the outputs with the generated test cases and their agreement with the outputs of other candidate programs. For both $\mathcal{B}^4$ and \textsc{CodeT}, we use the authors' code, with the generated input-output examples provided in the repository of~\citet{Mouxiang24}.
We also compare with the \emph{Original} baseline, i.e., the LLM which returns a program given a task. 
\Cref{tab:general_comparison} shows the results. \tool outperforms all baselines across all datasets and models. On average, its pass@1 exceeds $\mathcal{B}^4$ by $+13.0\%$, \textsc{CodeT} by $+16.6\%$, and \emph{Original} by $+40.7\%$. We assess the statistical significance of these improvements with McNemar's test~\cite{McNemar47}, suitable for comparing two algorithms on the same set of tasks with a binary success metric (e.g., pass@1). It confirms that \tool significantly outperforms each baseline on every dataset ($p$-value $< 0.001$).

\begin{table}[t]
\centering
\setlength{\tabcolsep}{2pt}  

\begin{tabular}{ll@{\hspace{0pt}}c@{\hspace{0pt}}cccc}  
\toprule
\textbf{Dataset} & \textbf{Model} & \textbf{\makecell{\texttt{ExPairT-} \\ \texttt{LLM}}} & \textbf{$\mathcal{B}^4$} & \textbf{\textsc{CodeT}} 
& \textbf{Orig} \\
\midrule
\multirow{5}{*}{\makecell{Human- \\ Eval}} 
    & Codex           & \textbf{91.3}   & 80.5  & 74.1  & 38.7 \\
    & CodeLlama       & \textbf{89.1}   & 66.6  & 62.5  & 40.3 \\
    & StarCoder       & \textbf{91.7}   & 70.5  & 65.4 & 37.6 \\
    & DeepSeek-Coder  & \textbf{91.6}   & 80.8  & 80.9  & 65.4    \\
    & CodeGen         & \textbf{94.1}      & 64.0  & 55.1  & 42.4    \\
\midrule
\multirow{5}{*}{MBPP} 
    & Codex           & \textbf{84.4}   & \textbf{84.4}  & 83.4  & 54.6 \\
    & CodeLlama       & \textbf{81.7}   & 76.2  & 76.7  & 50.0 \\
    & StarCoder       & \textbf{81.4}   & 72.7  & 72.4  & 43.6 \\
    & DeepSeek-Coder  & \textbf{78.9}   & 76.8  & 76.8  & 58.3 \\
    & CodeGen         & \textbf{82.6}   & 78.1  & 67.7  & 43.6 \\
\midrule
APPS 
    & Codex           & \textbf{81.4}      & 54.3  & 50.2  & 25.8   \\
\bottomrule
\end{tabular}
\caption{The pass@1 (\%) of \tool and baselines.}
\label{tab:general_comparison}
\end{table}

\paragraph{LLMs with Complex Reasoning.}
We next compare \tool to LLMs with strong reasoning capabilities:
\emph{OpenAI~o1-mini}~\cite{openaio1mini}, \emph{DeepSeek-R1}~\cite{deepseek}, and \emph{Gemini 2.5 Flash}~\cite{gemini}.
For each, we run \tool and use the LLM to generate $10$ candidate programs (for each task), to generate $5$ initial input examples, and as the oracle. 
We run this experiment on $150$ randomly selected tasks from \emph{APPS} and on \textsc{\emph{LiveCodeBench}}~\cite{jain2025livecodebench}, with $511$ tasks. \Cref{tab:all_oracle} shows the pass@1 of \tool and the original pass@1 of the LLM (i.e., the first program it returns). 
\tool improves the original pass@1 by $+32.8\%$ for \emph{OpenAI~o1-mini}, by $+20.4\%$ for \emph{DeepSeek-R1}, and by $+18.9\%$ for \emph{Gemini 2.5 Flash}. For both datasets, a McNemar's test confirms that \tool significantly outperforms the original pass@1 ($p$-value $< 0.001$).

\begin{table}[tb]
  \centering
    \setlength{\tabcolsep}{5pt} 
  \begin{tabular}{llcc}
    \toprule
    \textbf{Dataset}
      & \textbf{Model}
      & \textbf{\makecell{\texttt{ExPairT-} \\ \texttt{LLM}}}
      & \textbf{Original} \\
    \midrule
    \multirow{3}{*}{APPS}
      & OpenAI~o1‑mini     & \textbf{91.2} & 54.7 \\
      & DeepSeek‑R1        & \textbf{64.3} & 56.3 \\
      & Gemini 2.5 Flash   & \textbf{70.3} & 54.1 \\
    \midrule
    \multirow{3}{*}{\makecell{LiveCode- \\ Bench}} 
      & OpenAI~o1‑mini     & \textbf{95.8} & 66.7 \\
      & DeepSeek‑R1        & \textbf{88.4} & 55.7 \\
      & Gemini 2.5 Flash   & \textbf{92.9} & 71.4 \\
    \bottomrule
  \end{tabular}
    \caption{The pass@1 (\%) of \tool and LLMs with complex reasoning.}
    \label{tab:all_oracle}
\end{table}

\paragraph{Importance of Our Queries.}
Next, we study the importance of both pairwise membership and pairwise equivalence queries. 
We consider two variants. \emph{First,} $\tool\texttt{\_no\_equ}$, which relies solely on membership queries and skips the pairwise equivalence queries (\lineref{ln:sequ}--\lineref{ln:efoundfiff}).
That is, the programs are clustered based on their outputs for the initial inputs, then a cluster is selected by pairwise membership queries (i.e., by majority vote), and lastly the first program in it is returned.
\emph{Second,} $\tool\texttt{\_no\_mem}$, which skips the pairwise membership queries (\lineref{ln:smem}--\lineref{ln:argmax}) and instead selects the cluster with the most programs as $C^*$. \Cref{fig:no_mem_equ} presents the results. On average, on \emph{HumanEval}, the pass@1 of \tool is higher than $\tool\texttt{\_no\_equ}$ by $+7.7\%$ and $\tool\texttt{\_no\_mem}$ by $+24.6\%$. The McNemar's test confirms that \tool significantly outperforms both variants ($p$-value $< 0.001$).

\begin{figure}[tb]
  \centering
  \includegraphics[width=\linewidth]{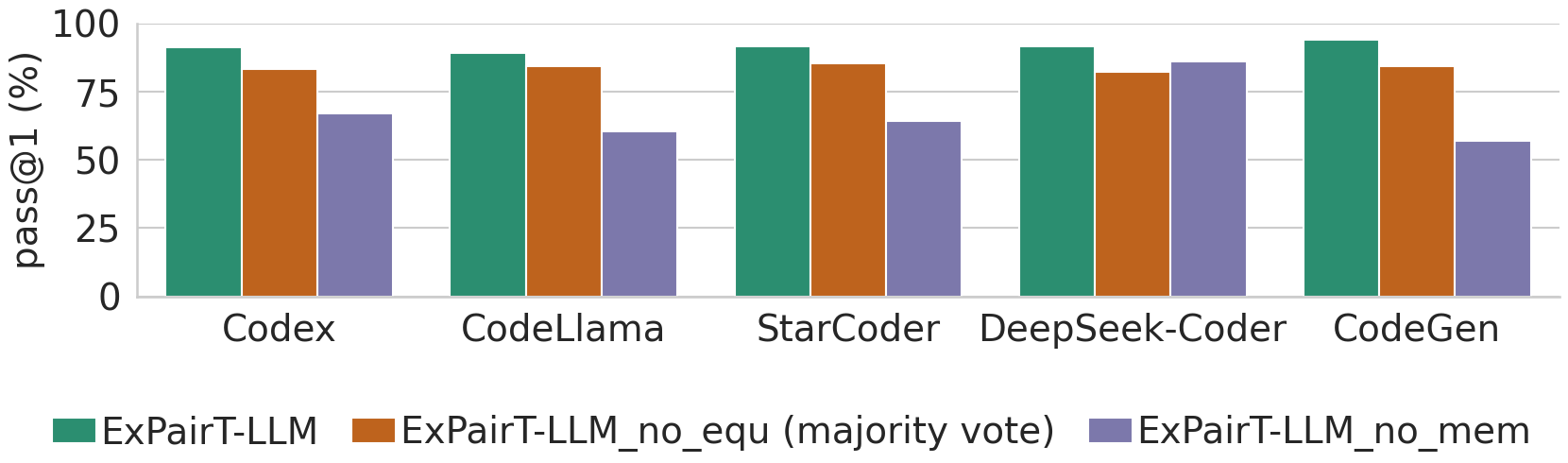}
  \caption{The pass@1 of \tool and variants without equivalence or membership queries, on \emph{HumanEval}.}
  \label{fig:no_mem_equ} 
\end{figure}

\begin{table}[t!]
  \centering
  \setlength{\tabcolsep}{8pt}

  \begin{tabular}{lcc}
    \toprule
    {\textbf{Model}} &
     \textbf{P. Mem.} & \textbf{P. Eq.} \\
    \midrule
    Codex           & 12.4 & 9.1 \\
    CodeLlama       & 20.8 & 8.2 \\
    StarCoder       & 16.6 & 8.2 \\
    DeepSeek-Coder  &  5.2 & 14.9 \\
    CodeGen         & 25.4 & 7.6 \\
    \bottomrule
  \end{tabular}
  \caption{The average \#queries on \emph{HumanEval}.}
  \label{tab:queries}
\end{table}

\paragraph{Total Queries.}  
Lastly,
\Cref{tab:queries} shows the number of queries \tool poses on \emph{HumanEval}, 
for five models and $|\mathcal{P}|=25$.
On average, it poses $16.1$ pairwise membership queries and $9.6$ pairwise equivalence queries, i.e., 
 5.3\% and 3.2\% of the worst-case bounds ${25 \choose 2}=300$ (Theorem \ref{thm:memb_worst_case}).
That is, the average total number of queries is $\approx |\mathcal{P}|$.

\section{Conclusion}
We present \tool, an exact learning algorithm for selecting the correct program. \tool poses pairwise membership and pairwise equivalence queries to an LLM oracle. If the LLM is accurate, \tool returns the correct program; otherwise, we present lower bounds on the probability of identifying the correct program. \tool outperforms the state-of-the-art, on average by $+13.0\%$, and improves the pass@1 of LLMs with strong reasoning capabilities by an average of $+24.0\%$.

\bibliography{aaai2026}

\newpage
\section{Proofs}\label{app:proofs}
\fta*
\begin{proof}
$C^*$ is initially $\mathcal{P}$ (\lineref{ln:init}). 
We show that, starting from the second iteration, 
the size of $C^*$ decreases by at least one in~\lineref{ln:argmax}.
Thus, in iteration $|\mathcal{P}|$, at the latest, $C^*$ contains a single program and \tool terminates.
At the end of every iteration, \tool either terminates or finds and validates a differentiating input. 
In the latter case, the next call to \texttt{cluster} returns at least two clusters which partition 
$C^*$, i.e., their sizes are smaller than the size of~$C^*$.
One of these clusters is set to~$C^*$ in~\lineref{ln:argmax}, and thus 
the claim follows.
\end{proof}

\ftb*
\begin{proof}
Let $p^* \in \mathcal{P}$ be a correct program for the task $t$. 
We prove by induction. 

\noindent\textbf{Base:}  
In the first iteration, ${C}^* = \mathcal{P}$. Thus, $p^* \in {C}^*$.

\noindent\textbf{Inductive Hypothesis:}  
Assume that $p^* \in {C}^*$ at the start of iteration $k$.

\noindent\textbf{Inductive Step:}  
We show that $p^* \in {C}^*$ at the start of iteration $k+1$. By the inductive hypothesis, $p^* \in {C}^*$ at the start of iteration $k$. 
In iteration $k$, the programs in ${C}^*$ are clustered 
by their outputs for the inputs in $\mathcal{I}$. Let $\bar{{C}}$ be the cluster containing $p^*$.
We show that $\bar{{C}}$ obtains the maximal score in Copeland's method. Thus, it is set to ${C}^*$ (\lineref{ln:argmax}) and
at the start of iteration $k+1$, $p^* \in {C}^*$.
Since the clusters differ by their outputs 
and since $p^*\in \bar{{C}}$, the score of $\bar{{C}}$ is $|C^*|-1$.
This follows since for every input in $\mathcal{I}$, the output produced by $p^*$ (which is identical to the outputs of all programs in $\bar{{C}}$ with respect to $\mathcal{I}$) is more suitable than the outputs of any other cluster.
Any other cluster ${C}_j \neq \bar{{C}}$ loses a point in the query comparing it to $\bar{{C}}$. Thus, $\bar{{C}}$ has the maximal score. 
\end{proof}

\ftc*
\begin{proof}
Assume ${C}^*$ contains an incorrect program $p'$.
By~\Cref{lem:mem}, ${C}^*$ contains a correct program $p^*$, which is not equivalent to $p'$.  
In \lineref{ln:sequ}--\lineref{ln:efoundfiff},
\tool compares the program ${C}^*[1]$ with every other program in $C^*$. 
Assume in contradiction that the oracle does not return a differentiating input for any pairwise equivalence query.
Namely, it determines that ${C}^*[1]$ is equivalent to $p^*$ and $p'$, implying that $p^*$ is equivalent to $p'$ -- contradiction. 
\end{proof}

\ftd*
\begin{proof}
By~\Cref{lem:termination}, \tool terminates.
In that iteration, $C^*$ contains one program or multiple equivalent programs.
This follows by~\Cref{lem:equ}, since in iterations that $C^*$ contains nonequivalent programs, the oracle finds a differentiating input,
and thus \tool continues to another iteration. 
When it terminates, \tool returns the first program in $C^*$.
By~\Cref{lem:mem}, in this iteration, $C^*$ contains a correct program. Thus, it returns a correct program.
\end{proof}

\fth*

\begin{proof}
We define two random variables:
\begin{itemize}
    \item \(X\): the number of pairwise membership queries in which $C_{n+1}$ wins.
    Since $C_{n+1}$ is compared once with each cluster, wins with probability $p$, and the queries are independent: 
$X \sim \mathrm{Binomial}(n, p)$.
    \item \(Y_i\), for \(i \in [n]\): the number of pairwise membership queries where $C_i$ wins an incorrect cluster.
Since $C_i$ is compared once with each cluster, wins with probability $0.5$, and the queries are independent: 
$
    Y_i \sim \mathrm{Binomial}(n-1, 0.5)
$.
\end{itemize}

let $B$ be the event that the maximum over all $Y_i$ is at most~$j$, that is: 
$\max_{1 \le i \le n} Y_i \,\le\,j$. 
By the law of total probability:
$$
    \mathbb{P}(A)
    =
    \mathbb{P}(A \;\bigm|B) \cdot 
    \mathbb{P}(B)+
    \mathbb{P}(A \;\bigm|\overline{B}) \cdot 
    \mathbb{P}(\overline{B})
$$

Since probabilities are nonnegative, we lower bound the second term by 0.

\noindent
\textbf{Bounding }
\(\displaystyle \mathbb{P}\bigl(A \mid B\bigr)\)\textbf{:}
Given $B$, for $C_{n+1}$ to obtain the maximal score, it is sufficient that:
$
X \;\ge\; j \;+\; 2  
$.
Since \(X \sim \mathrm{Binomial}(n, p)\), we obtain:
\begin{align*}
    \mathbb{P}(A\bigm| B)\ge 
 \mathbb{P}\bigl(X \,\ge\, j + 2\bigr)
    &\;=\;
    \sum_{k=\lceil j \rceil + 2 }^n
    \binom{n}{k}\, p^k\, q^{\,n-k}.
\end{align*}

\noindent
\textbf{Bounding }
\(\displaystyle \mathbb{P}\bigl(B\bigr)\)\textbf{:}
By the complement rule, $\mathbb{P}\bigl(B\bigr)=1-\mathbb{P}\bigl(\overline{B}\bigr)$,
where $\overline{B}$ is
 $\max_{1 \le i \le n} Y_i >  j$.

By the union bound and since \(Y_i\sim \mathrm{Binomial}(n-1, 0.5)\):
\[
    \mathbb{P}\bigl(\overline{B}\bigr)
    \;\le\;
    \sum_{i=1}^{n}
\mathbb{P}\bigl(Y_i > j \bigr)
    \;=\;
    n \cdot
    \sum_{k=\lfloor j \rfloor + 1}^{\,n-1}
    \binom{n-1}{k}
    \Bigl(\tfrac{1}{2}\Bigr)^{n-1}.
\]

Hence,
$
    \mathbb{P}(B)
    \;\ge\;
    1
    \;-\;
    n \,\Bigl(\tfrac{1}{2}\Bigr)^{n-1}
\sum_{k=\lfloor j\rfloor + 1}^{\,n-1} \binom{n-1}{k}.
$

Combining all bounds completes the proof.
\end{proof}

\ftg*

\begin{proof}
Let \(B\) be the event that the first program \({C}^*[1]\) is correct.
By the law of total probability,
$
    \mathbb{P}(A) 
    =
    \mathbb{P}(A \mid B)\,\mathbb{P}(B)
    +
    \mathbb{P}(A \mid \overline{B})\,\mathbb{P}(\overline{B})
$.
Since the order of the programs in \({C}^*\) is uniformly random,
$
    \mathbb{P}(B) 
    \;=\; 
    \frac{\lvert {C}^*\rvert - d}{\lvert {C}^*\rvert}
    \text{\ ;}\quad
    \mathbb{P}(\overline{B}) 
    \;=\;
    \frac{d}{\lvert {C}^*\rvert}
$.
Given~$B$, the probability of the oracle returning a differentiating input is the complementary probability
of not returning a differentiating input for all incorrect programs. Since the queries are independent, 
\[
    \mathbb{P}(A \,\mid\, B)
    \;=\;
    1-\mathbb{P}(\overline{A} \,\mid\, B)
    =1-(1-p)^{\,d}.
\]
Similarly, \[
\begin{aligned}
\mathbb{P}(A \mid \overline{B})
&= 1- \mathbb{P}(\overline{A} \mid \overline{B}) \\
&= 1-(
    (1-p)^{\lvert {C}^*\rvert - d} 
    \cdot
    (1-p')^{d - 1}
  )
\end{aligned}
\]
Substituting these terms into the law of total probability completes the proof. 
\end{proof}

\fte*
\begin{proof}
Proof for the pairwise membership queries:
We show that every two programs $p_i,p_j\in \mathcal{P}$ can be compared by a pairwise membership query, defined over their
containing clusters $C_i$ and $C_j$, at most once.  
Thus, the maximal number of pairwise membership queries is ${|\mathcal{P}| \choose 2}$.
Let $p_i$ and $p_j$ be programs that are compared by a pairwise membership query in some iteration.
Since pairwise membership queries are between different clusters, $p_i$ and $p_j$ are not in the same cluster.
Denote their clusters $C_i$ and $C_j$, respectively. 
At the end of this iteration, one cluster is selected as $C^*$, and all other clusters (including the programs they contain)
are discarded. Thus, at least $p_i$ or $p_j$ cannot be compared to any other program in future iterations.

Proof for the pairwise equivalence queries:
Consider $C^*$ in~\lineref{ln:flagf}.
In the first iteration, $|C^*|\leq |\mathcal{P}|$ and starting from the second iteration, $|C^*|$ decreases by at least one, 
since the differentiating input leads to splitting the previous $C^*$. 
Since in every iteration \tool poses up to $|C^*|-1$ pairwise equivalence queries,
it poses at most $\sum_{m=1}^{| \mathcal{P}|-1} m = \frac{|\mathcal{P}|(|\mathcal{P}|-1)}{2}={|\mathcal{P}| \choose 2}$ pairwise equivalence queries.
\end{proof}
\section{LLM Prompt Examples for Pairwise Queries}
\label{app:ex_query}

\Cref{fig:example_mem_query} and \Cref{fig:example_equ_query} show prompt examples provided to the LLM oracle for pairwise membership and equivalence queries, respectively. The text in \textbf{black} is the general prompt template, which explains the task and specifies the desired response format. The text in \textbf{\textcolor{darkgreen}{green}} is the coding task $t$.
For the pairwise membership query (\Cref{fig:example_mem_query}), the text in \textbf{\textcolor{red}{red}} highlights the inputs differentiating the two programs, and the text in \textbf{\textcolor{blue}{blue}} shows the executed outputs of the two clusters for those inputs. 
For the pairwise equivalence query (\Cref{fig:example_equ_query}), the text in \textbf{\textcolor{purple}{purple}} is the programs compared for semantic equivalence. If the programs are not equivalent, the LLM is expected to provide a differentiating input.
\FloatBarrier 

\begin{figure}[t]
\centering
\begin{adjustbox}{center}
\begin{tcolorbox}[queryexample, title=Example: Pairwise Membership Query]
You are provided with the partial docstring of a Python function and multiple input-output examples from two different programs (Program 1 and Program 2). Your task is to analyze these examples and determine which program's outputs better align with the function's intended behavior as described in the docstring. **Please provide only the classification as a single term: "Program 1" or "Program 2". Do not include any additional text or explanations.**
Here is the information provided:
\\\\
- **Docstring:** 
\\
\textcolor{darkgreen}{\texttt{
def max\_difference(test\_list):
\begin{quote}
"""\\Write a function to find the maximum difference between available pairs in the given tuple list.\\"""
\end{quote}
}}

- **Input-Output Examples:**
\\
Example 1:\\
For Input: \textcolor{red}{list1=\textbf{\boldmath$\big[$}(3, 5), (1, 7), (3, 10), (1, 2)\textbf{\boldmath$\big]$}}\\
Output of Program 1:\\
\textcolor{blue}{-1}\\
Output of Program 2:\\
\textcolor{blue}{7}\\

Example 2:\\
For Input: \textcolor{red}{list1=\textbf{\boldmath$\big[$}(6, 4), (2, 17), (9, 13), (11, 12)\textbf{\boldmath$\big]$}}\\
Output of Program 1:\\
\textcolor{blue}{2}\\
Output of Program 2:\\
\textcolor{blue}{15}
\end{tcolorbox}
\end{adjustbox}
\caption{An example of a prompt for a pairwise membership query.}
\label{fig:example_mem_query}
\end{figure}

\newpage
\begin{figure*}[hbtp]
\centering
\begin{adjustbox}{center}
\begin{tcolorbox}[queryexample, title=Example: Pairwise Equivalence Query]
You are given the following partial docstring of a Python function:
\\\\
\textcolor{darkgreen}{
\texttt{def max\_difference(test\_list):
\begin{quote}
"""\\Write a function to find the maximum difference between available pairs in the given tuple list.\\"""
\end{quote}
}}
and two Python programs implementing this function:\\
\\
**Program 1:**\\

\textcolor{purple}{
\texttt{def max\_difference(test\_list):}
\begin{quote}
\texttt{res = max([(a - b) for a, b in test\_list])\\
return res}
\end{quote}
}

**Program 2:**\\

\textcolor{purple}{
\texttt{def max\_difference(test\_list):}
\begin{quote}
\texttt{res = max([abs(b - a) for a, b in test\_list])\\
return res}
\end{quote}
}

**Task:** Find an input example such that when this input is provided to both programs, they produce different outputs. \\

**Important:** The input should be **valid according to the function's description**.\\

**Output Format:** If you find such an input, please return it in the format:\\
parameter\_name1=input\_value1, \\parameter\_name2=input\_value2, ...\\

If you cannot find such an input, please return the term "NO\_DIFF".\\

**Important:** Provide **only** the input in the specified format or the term "NO\_DIFF", without any additional explanations or output.
\end{tcolorbox}
\end{adjustbox}
\caption{An example of a prompt for a pairwise equivalence query.}
\label{fig:example_equ_query}
\end{figure*} 

\end{document}